\documentclass[twoside,11pt]{article}

\usepackage{blindtext}

\usepackage{jmlr2e}

\usepackage{bbm}
\usepackage{amsmath}%
\usepackage{amsfonts}%
\usepackage{amssymb}%
\usepackage[svgnames]{xcolor}%
\usepackage{tikz}
\usepackage{array}

\usepackage[ruled,vlined,noend]{algorithm2e}

\usepackage{algpseudocode}

\usepackage{subcaption}
\usepackage{graphicx}%

\DeclareMathOperator*{\argmax}{arg\,max}

\newcommand{\redcircle}{\raisebox{-0.8pt}{\tikz{\draw[color=DarkRed, fill=red,  thick](1mm,0) circle (1.2mm);}}}

\newcommand{\greencircle}{\raisebox{-0.8pt}{\tikz{\draw[color=DarkGreen, fill=LightGreen,  thick](1mm,0) circle (1.2mm);}}}

\definecolor{myorange}{rgb}{0.9,0.6,0}
\definecolor{myblue}{rgb}{0.35,0.7,0.9}
\definecolor{mygreen}{rgb}{0,0.6,0.5}
\definecolor{myred}{rgb}{0.8,0.4,0}

\newcommand{\blackline}{\raisebox{2pt}{\tikz{\draw[-,black,solid,line width = 1.5pt](0,0) -- (4mm,0);}}}
\newcommand{\orangeline}{\raisebox{2pt}{\tikz{\draw[-,myorange,solid,line width = 1.5pt](0,0) -- (4mm,0);}}}
\newcommand{\blueline}{\raisebox{2pt}{\tikz{\draw[-,myblue,solid,line width = 1.5pt](0,0) -- (4mm,0);}}}
\newcommand{\greenline}{\raisebox{2pt}{\tikz{\draw[-,mygreen,solid,line width = 1.5pt](0,0) -- (4mm,0);}}}
\newcommand{\redline}{\raisebox{2pt}{\tikz{\draw[-,myred,solid,line width = 1.5pt](0,0) -- (4mm,0);}}}

\usepackage{lastpage}

\ShortHeadings{Memoryless Policy Iteration for Episodic POMDPs}{van Zuijlen and Antunes}
\firstpageno{1}

\begin{document}

\title{Memoryless Policy Iteration for Episodic POMDPs}

\author{\name Roy van Zuijlen \email r.a.c.zuijlen@tue.nl \\
       \name Duarte Antunes \email d.antunes@tue.nl \\
       \addr Control Systems Technology Section, Department of Mechanical Engineering\\
       Eindhoven University of Technology\\
       Eindhoven, The Netherlands}

\maketitle

\begin{abstract}

Memoryless and finite-memory policies offer a practical alternative for solving partially observable Markov decision processes (POMDPs), as they operate directly in the output space rather than in the high-dimensional belief space. However, extending classical methods such as policy iteration to this setting remains difficult; the output process is non-Markovian, making policy-improvement steps interdependent across stages. We introduce a new family of monotonically improving policy-iteration algorithms that alternate between single-stage output-based policy improvements and policy evaluations according to a prescribed periodic pattern. We show that this family admits optimal patterns that maximize a natural computational-efficiency index, and we identify the simplest pattern with minimal period. Building on this structure, we further develop a model-free variant that estimates values from data and learns memoryless policies directly. Across several POMDPs examples, our method achieves significant computational speedups over policy-gradient baselines and recent specialized algorithms in both model-based and model-free settings.

\end{abstract}


\section{Introduction}

\par Finding an optimal policy for a partially observable Markov decision problem (POMDP) is PSPACE-complete~\cite{tsitsiklis:87}. In the episodic case, the optimal value function over the belief simplex is piecewise linear but its complexity grows exponentially with the horizon~\cite{sondik:73}. As a result, a large body of work has focused on suboptimal but practical solution methods; see the survey papers~\cite{spaan2012partially},~\cite{surveyPOMDPs2}. A prominent class of such methods restricts the policy search to memoryless or finite-memory policies that depend on a short window of past output. However, once cast in the output space of observations, this system becomes non-Markovian, and computing an optimal memoryless policy is itself an NP-hard problem~\cite{littman1994memoryless},~\cite{vlassis:12}.
\par Approximate model-based and model-free techniques have therefore been explored to compute memoryless policies, long regarded as an open problem~\cite{AzizzadenesheliLazaricAnandkumar2016_OpenProblem_MemorylessPOMDPs}. Since finite-memory policies can be represented as memoryless policies through state augmentation~\cite{Littman:1993} most approaches below, and the methods proposed in this paper, apply to both settings.  
\par In the context of model-based approaches,~\cite{cohen2023future} propose a rolling-horizon method for computing memoryless policies using mixed-integer linear programming (MILP). Across a large benchmark of POMDPs, their MILP-based policy is within 20\% of the optimal POMDP value in 60$\%$ of the problems. ~\cite{MuellerMontufar2022} prove that the state-action frequencies and the expected cumulative rewards are rational functions of the policy and use linear optimization subject to polynomial constraints to find a memoryless policy.  \cite{SteckelmacherRoijersHarutyunyan2018_MemorylessOptionsPOMDPs} show that finite-state controllers, which include memoryless policies, can be expressed with options via the so called option-observation initiation sets method, which simplifies the search for memoryless policies. More recently,~\cite{indian:24} study problems modelled with the so-called agent state (a generally non-Markovian construct that subsumes finite output windows) and show that periodic time-varying policies can outperform stationary ones.
\par In the context of model-free methods,~\cite{Azizzadenesheli2016ReinforcementLO} propose a reinforcement learning (RL) method that relies on estimating the parameters of a POMDP using spectral methods, an oracle to compute an optimal memoryless policy, and a smart exploration strategy. Related to this approach,~\cite{neurips:2020} propose a sample-efficient RL method for episodic undercomplete POMDPs, where the number of observations exceeds the number of states; data is used to estimate a model and efficiently guide exploration. \cite{21Mahajan22} present policy gradient based online RL algorithms for POMDPs which learn an agent information state representation using multi-scale stochastic gradient descent. \cite{SinhaMahajan2024_AgentStateBasedPolicies} discuss how ideas from approximate information state for computing agent-based (and memoryless) policies can be used to improve Q-Learning and actor-critic algorithms for POMDPs.~\cite{Perkins:2002} proposes a Q-learning type algorithm relying on stochastic optimization. Systems with non-Markovian rewards are considered in~\cite{13camacho17}, in which a reward shaping approach is proposed to handle temporal logic specifications. 
Policy gradient methods, which rely on parametrizations of stochastic policies and optimize its parameters via gradient estimates, have also been proposed in the context of memoryless policies~\cite{15Jaakkola94},~\cite{BaxterBartlett2000_POMDPsDirectGradient}.  Several recent papers propose deep RL approaches for POMDP~\cite{cai2024provable},~\cite{ishida2024soap},~\cite{arcieri2024pomdp},~\cite{rozek2024partially}. 
\par A central challenge in all these approaches is the lack of the Markov property in the output space, which prevents direct application of standard MDP planning and RL techniques such as, most notably, policy iteration. Only a few papers attempt to extend policy iteration to the memoryless setting.~\cite{LiYinXi2011_MemorylessPOMDPs_AverageReward} propose a policy-iteration scheme for the average-reward case, but each iteration requires an exhaustive search over the policy space, which is computationally prohibitive except for very small problems or heavily restricted policy classes. Memoryless policies are special cases of finite-state controllers, for which policy iteration has been proposed~\cite{Hansen1997_ImprovedPolicyIterationPOMDPs,hansen2}. However, the algorithms in~\cite{Hansen1997_ImprovedPolicyIterationPOMDPs,hansen2} work explicitly in the belief space and they do not scale up well with respect to the size of the problem~\cite{FSC:99}. In the related setting of general non-Markovian systems,~\cite{15Jaakkola94} propose an average-cost algorithm reminiscent of policy iteration, but their policy improvement step performs a gradient update over the probabilities of a randomized policy, based on cost estimates from policy evaluation. In contrast, we focus on deterministic memoryless policies and depart from gradient-based approaches, which are fundamentally different from policy iteration and, as we shall see, substantially less computationally efficient for the problems we target.  
\par Stochastic memoryless policies have been reported to outperform deterministic ones in several non-Markovian contexts~\cite{barto1983neuronlike},~\cite{loch1998using},~\cite{williams1998experimental},~\cite{krishnamurthy2016contextual}. However, the empirical results in~\cite{cohen2023future} question the generality of this claim, and~\cite{Bertsekas:08} shows that deterministic finite-history policies can approximate the optimal deterministic POMDP policy arbitrarily well. In our view, the structure of the output space plays a crucial role in comparing deterministic and stochastic policies. Deterministic policies, on which we focus in this paper, are particularly appealing in applications where predictability and safety are essential.

In this paper, we introduce a new family of policy-iteration algorithms for computing deterministic memoryless policies in episodic POMDPs. Each algorithm in this family alternates between single-stage, output-based policy improvements and policy-evaluation steps following a prescribed periodic pattern that determines which stages are improved and when. The non-Markovian nature of the output process introduces two key departures from classical MDP policy iteration.
First, policy evaluation must compute not only the state–action Q-values but also the state-distribution trajectory, since both quantities influence output-based decisions.
Second, policy improvements cannot be applied simultaneously across all stages: improvements at one stage affect the value of subsequent stages, making stagewise improvements inherently interdependent.

We analyse how the choice of periodic pattern influences a natural index of computational efficiency in the time-invariant setting, and identify the patterns that maximize this index. Among them, we advocate the simplest optimal pattern, namely the one with minimal period. Under mild conditions, we show that every algorithm in this family converges monotonically to a locally optimal memoryless policy. Building on this structure, we further introduce a model-free variant that learns memoryless policies directly from data. Experiments demonstrate substantial computational gains over policy-gradient baselines and recent specialized methods, in both model-based and model-free settings.

\par The remainder of the paper is organized as follows. Section~\ref{sec:II} formulates the problem. Section~\ref{sec:III} proposes the general class of algorithms and presents the main results. Section~\ref{sec:IV} tackles model-free methods. Section~\ref{sec:V} provides numerical examples. Section~\ref{sec:VI} discusses future work and provides final remarks. 

\par \textit{Notation}: The number of elements in set $\mathcal{A}$ is denoted by $|\mathcal{A}|$. $\|\cdot\|^2_\Lambda$ denotes the weighted squared norm: $\|x\|^2_\Lambda = x^\top \Lambda x$. The indicator function $\mathbbm{1}\lbrace E \rbrace$ equals 1 if the statement $E$ is true and 0 otherwise.

\section{Problem formulation}\label{sec:II}
Consider a standard Partially Observable Markov Decision Process. The states and actions at time $t \in \mathcal{T} := \lbrace 0,1,...,T\rbrace$, $T \in \mathbb{N}$, are denoted by $S_t \in \mathcal{S}$ and $A_t \in \mathcal{A}$, where $\mathcal{S}$ and $\mathcal{A}$ are finite sets. The dynamics are characterized by
\begin{equation*}
    p_t(s^\prime|s,a) = \text{Pr}\lbrace S_{t+1} = s^\prime \mid S_{t} = s, A_t = a \rbrace.
\end{equation*}
 The initial state is assumed to be distributed according to a known probability distribution $\mu_0(s) = \text{Pr}\lbrace S_0 = s \rbrace$. Observations at time $t$ are denoted by $O_t \in \mathcal{O}$, where $\mathcal{O}$ is a finite set. The observation model at time $t$ is given by
\begin{equation*}
    q_t(o|s) = \text{Pr}\lbrace O_t = o \mid S_{t} = s \rbrace.
\end{equation*} 
\par A finite time horizon (episodic task) is considered.
The reward at time $t$ is denoted by $R_t$ and the (stage-dependent) expected reward is a function of the current state and action,
\begin{equation*}
    r_t(s,a) = \mathbb{E}\left\lbrack R_t | S_t = s, A_t = a \right\rbrack.
\end{equation*}
\par For POMDPs the optimal policy is in general a function of the complete history of observations and actions $H_t := \lbrace A_0, O_0, ..., A_{t-1}, O_{t} \rbrace$, but it can also be written as a function of the belief state  $b_t(s) = \text{Pr}\lbrace S_t = s | H_t \rbrace$~\cite{sondik:73}.
History dependent policies scale poorly with time and belief state dependent policies scale poorly with the state dimension $n$, making both approaches computationally impractical. Motivated by this, in this paper, we focus instead on finding (suboptimal) time-varying policies $\lbrace {\pi}_t : \mathcal{O} \rightarrow \mathcal{A}\rbrace$ that only depend on the current observation
\begin{equation*}
    {\pi}_t(a|o) = \text{Pr}\lbrace A_t = a | O_t = o \rbrace.
\end{equation*}
With some abuse of notation, we use ${\pi}_t:\mathcal{O}\rightarrow \mathcal{A}$, where $\pi_t(o)$ is an action, to denote deterministic observation-based policies and probability ${\pi}_t(a|o) \in [0,1]$ for their equivalent degenerate probabilistic representation
    \begin{equation}\label{eq:pidet}
        {\pi}_t(a|o) := \mathbbm{1}\lbrace a = {\pi}_t(o)\rbrace.
    \end{equation}
 More specifically, we are interested in \textit{deterministic} policies: $\lbrace {\pi}_t : \mathcal{O} \rightarrow \mathcal{A} \rbrace_{t=0}^{T-1}$ that maximize the expected episodic return
\begin{equation}\label{eq:L}
L^{\pi}= \mathbb{E}\left \lbrack \sum_{t=0}^{T-1} R_t + V_T(s_T) \right\rbrack.
\end{equation}
where $V_T(s_T) $ is the terminal cost. These policies scale well, are simple to implement, and give predictable behaviour. 
In the next section, we propose a model-based policy iteration scheme, that monotonically improves an initial deterministic observation-based policy to a (local) optimum.

\section{Model-Based Policy Iteration for Memoryless Policies}\label{sec:III}
 
A general class of memoryless policy iteration algorithms for episodic POMDPs is presented in Section~\ref{sec:III-A} where it is shown that they converge monotonically.  A proposed computationally efficient instance of this class is presented in Section~\ref{sec:III-B}.  

\subsection{Memoryless Policy Iteration with Periodic Stage Updates}\label{sec:III-A}

\par Consider a memoryless policy ${\pi}_t(a|o)$, let  $Q^\pi_T(s,a) =  V_T(s)$ and let the state-action value at time $t$ be
\begin{equation}\label{eq:Q_eval_obs_based}
    Q^{{\pi}}_t(s,a) = r_t(s,a) + \sum_{s^\prime} p_t(s^\prime \mid s,a) \Big(\sum_{o^\prime} q_{t+1}(o^\prime \mid s^\prime) \Big(\sum_{a^\prime} {\pi}_{t+1}(a^\prime \mid o^\prime) Q^{{\pi}}_{t+1}(s^\prime,a^\prime)\Big)\Big)
\end{equation}
where the inner summation boils down to $Q^{{\pi}}_{t+1}(s^\prime,\pi_t(o))$ for deterministic policies~\eqref{eq:pidet}. These state-action values can be computed backward in time from $t=T-1$ until $t=0$. Moreover, let
 the observation-action value be
\begin{equation}\label{eq:Q_function_obs_based}
    \Bar{Q}^{{\pi}}_t(o,a) = \sum_s Q^{{\pi}}_t(s,a) \alpha_{t}(s\mid o)
\end{equation}
where $\alpha_t(s|o) = \text{Pr}\lbrace S_t = s | O_t = o \rbrace$ is the posterior distribution of the state conditioned on observation $o$. Using Bayes' rule:
\begin{equation}\label{eq:alpha_posterior}
    \alpha_t(s|o) = \frac{\text{Pr}\lbrace O_t = o | S_t = s \rbrace \mu_t(s)}{\gamma_t(o)}
\end{equation}
with $ \gamma_t(o) = \text{Pr}\lbrace O_t = o \rbrace=\sum_s \text{Pr}\lbrace O_t = o | S_t = s \rbrace \mu_t(s),$ 
and prior $\mu_t(s) = \text{Pr}\lbrace S_t = s \rbrace$. The state distribution evolves under $\pi$ as
\begin{equation}\label{eq:mu} \mu_{t+1}(s) = \sum_{s^\prime} \Big(\sum_{o^\prime} \Big(    \sum_{a^\prime}{p_t(s|s^\prime,a')\pi}_{t}(a^\prime|o^\prime)\Big)q_t(o^\prime|s^\prime)\Big)\mu_t(s^\prime)\end{equation}
given the initial state distribution $\mu_0(s)$, where the inner summation boils down to $p_{t}(s|s^\prime,\pi_t(o^\prime))$ for deterministic policies~\eqref{eq:pidet}.
\par The proposed class of policy iteration algorithm is parametrized by an arbitrary periodic sequence $\tau_\ell\in\mathcal{T}$, $\ell \in \{0,1,\dots,\}$ with arbitrary period $M$. 
\begin{enumerate}
  \item[] \hspace{-0.8cm} Given a deterministic initial policy $\pi^0$ and an $M$-periodic schedule ${\tau_\ell}$, set $\ell=0$.
    \item[1.] \textbf{Policy evaluation:} Compute $\bar{Q}^{\pi^\ell}_t(o,a)$ for  $t =\tau_\ell$.
    \item[2.] \textbf{Policy improvement:} Compute
    \begin{equation*}
{\pi}^{\ell+1}_t(o) = \argmax_a \bar{Q}_t^{{\pi}^{\ell}}(o,a) \text{ for  }t=\tau_\ell.
    \end{equation*}
and set ${\pi}^{\ell+1}_t(o)={\pi}^{\ell}_t(o)$ for $t\neq \tau_\ell$.  Set $\ell \rightarrow \ell+1$ and repeat 1 and $2$ until the policy does not change over $M$ consecutive policy improvement steps.
\end{enumerate}
\par We call this algorithm for a fixed sequence $\{\tau_\ell\}$ memoryless policy iteration. Note that to compute $\bar{Q}^{\pi^\ell}_t(o,a)$  in step 1 for a new policy obtained in step 2 one needs to compute the state-action value and the state distribution at time $t$. This can be done with a forward pass of~\eqref{eq:mu} from $t=0$ until $t=\tau_\ell$ and a backward pass of~\eqref{eq:Q_eval_obs_based} from $t=T-1$ until $t=\tau_\ell$. 
\par In fully observable MDPs \big($\mathcal{O}=\mathcal{S}$, $q_t(o|s)=\mathbbm{1}\{o=s\}$\big), one has
\[
\alpha_t(s|o)=\mathbbm{1}\{s=o\},\qquad
\bar{Q}_t^\pi(o,a)=Q_t^\pi(o,a),
\]
and $\mu_t$ plays no role in policy improvement. Therefore, standard policy iteration updates all stages simultaneously.
In contrast, under partial observability, a change in policy at stage $t=\tau_\ell$ propagates through $\mu_t$, altering $\bar{Q}_{t^\prime}^\pi(o,a)$ at other stages. Thus, the observation-action values of interest (namely that for stage $\tau_{\ell+1}$) need to be updated immediately after a stage policy improvement is carried out, otherwise the policy improvement step at stage $\tau_{\ell+1}$ does not necessarily lead to a cost improvement.


 We assume that the periodic sequence ${\tau_\ell}$ is onto, ensuring that every stage is eventually updated. Moreover, we assume that $\tau_{\ell+1}\neq \tau_{\ell}$ as two consecutive policy improvements at the same stage would be redundant.  Under these conditions, analogously to classical policy iteration, we can guarantee monotonic improvement of the expected episodic return and convergence to a fixed policy. Unlike the fully observable MDP setting, however, the limiting policy need not be globally optimal due to the intrinsic non-Markovian structure of the output process. We say that a memoryless policy is locally optimal if no single-stage, observation-based deterministic deviation can further improve the expected initial-stage value. The formal statement is provided next and the proof is provided in the appendix.

\par 

\begin{theorem}\label{th:monotonic_convergence}
    Let $\pi^\ell$ be the policy produced by memoryless policy iteration under an arbitrary periodic onto sequence ${\tau_\ell}$ satisfying $\tau_{\ell+1}\neq \tau_\ell$, and denote its expected episodic return by $L^{\pi^\ell}$ as in~\eqref{eq:L}. Then
    $$L^{\pi^{\ell+1}} \geq L^{\pi^{\ell}}\text{ for all } \ell \in \{0,1,2,\dots\}$$ 
    and the algorithm terminates at a locally optimal memoryless policy.
\end{theorem}

There are many possibilities to choose the periodic sequence $\tau_\ell$, such as forward sweeps:
\begin{equation*}
    \tau = \lbrace 0,1,...,T-1,0,1,...,T-1,... \rbrace
\end{equation*}
or, more in style of Dynamic Programming, backward sweeps:
\begin{equation*}
    \tau = \lbrace T-1,T-2,...,1,0,T-1,T-2,...,1,0,... \rbrace.
\end{equation*}
Each sequence leads, in general, to a distinct \textit{local} optimum.  The backward sweep leads to a policy-iteration algorithm similar to that for MDPs. In fact, as one sweeps backwards only the state-action value needs to be updated and the state distribution update can be omitted as it is trivially obtained for MDPs. However, in POMDPs this is not necessarily the most efficient ordering. Optimized patterns are discussed next.


\subsection{Efficient memoryless policy iteration}\label{sec:III-B}

A key observation for reducing the computational effort of memoryless policy iteration under a given sequence ${\tau_\ell}$ is the following. Suppose a policy improvement is performed at stage $\tau_0$, and the next update occurs at $\tau_1>\tau_0$. Then only the state distributions $\mu_{\tau_0+1},\dots,\mu_{\tau_1}$ must be recomputed, since the change at $\tau_0$ only affects the forward propagation of $\mu_t$. In contrast, the state–action values $Q_t^{\pi^1}(s,a)$ for $t\ge \tau_0$ need not be recalculated until a later improvement is executed, because they can be obtained by backward recursion from~\eqref{eq:Q_eval_obs_based}. Consequently, computing $\bar{Q}_{\tau_1}^{\pi^1}(o,a)$ and performing the improvement at stage $\tau_1$ only requires updating $\mu_{\tau_0+1},\dots,\mu_{\tau_1}$.

A symmetric argument holds when the next improvement occurs at $\tau_1’<\tau_0$. In that case, only the values $Q_t^{\pi^1}$ for $t=\tau_0-1,\dots,\tau_1’$ must be recomputed backward, while the distributions $\mu_t$ for $t\le \tau_0$ remain unaffected, as they depend solely on forward propagation from the initial state.  These findings are summarized and illustrated in Figure~\ref{fig:Figure_mu_Q}.
\par Building on this insight, we define an index that quantifies the number of policy evaluation operations per period required by a  sequence $\{\tau_\ell\}$  in the time-invariant case ($p_t,q_t,r_t$ do not depend on $t$) 
\begin{equation}\label{eq:periodic_cost_C}
    C = \frac{1}{M}{\sum_{\ell=0}^{M-1} {w_\mu \max(\tau_{\ell+1}-\tau_\ell,0)+w_Q \max(0,\tau_{\ell}-\tau_{\ell+1})}}
\end{equation}
where
\begin{itemize}
    \item $w_\mu$ - computational cost of one stage state distribution update~\eqref{eq:mu}.
    \item $w_Q$ - computational cost of one stage state-action update~\eqref{eq:Q_eval_obs_based}.
\end{itemize}
Although $w_\mu$ and $w_Q$ are often comparable (both involve summation over states and observations), distinguishing them enables a more general efficiency analysis. The index $C$ represents the total number  of update operations required to perform $M$ policy improvements - hence smaller $C$ lead to faster learning, as improvements increase performance while evaluations are auxiliary. Note that this performance index assumes that the computational cost of each policy-evaluation step, as well as the expected value improvement of each policy-improvement step, is identical across all stages. The conclusions that follow therefore apply only to problems that (approximately) satisfy these assumptions. They do not extend to settings in which certain stages yield systematically larger value gains, or where policy-evaluation costs vary significantly across stages (e.g., time-varying problems). 
\par The next result identifies the most computationally efficient periodic sequences. 
\begin{figure}[t]
    \centering
    \includegraphics[width=0.8\linewidth]{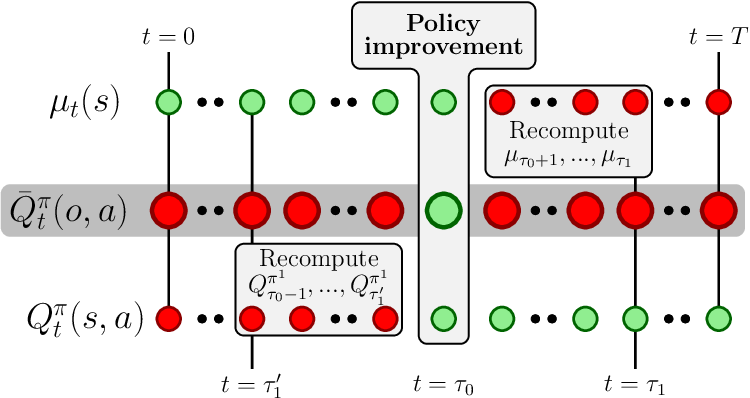}
    \caption{If the policy is improved at $t=\tau_0$, $\bar{Q}^{\pi^1}_t(o,a)$ is outdated for $t < \tau_0$ since $Q^{\pi^1}_t(s,a)$ changes, and for $t > \tau_0$ since $\mu_t(s)$ changes. If the next improvement occurs at $\tau_1$, where $\tau_1 > \tau_0$, only $\mu_{\tau_0 + 1},...,\mu_{\tau_1}$ have to be evaluated. Moreover, if the next improvement occurs at $\tau_1^\prime$, where $\tau_1^\prime < \tau_0$, only $Q^{\pi^1}_{\tau_0 - 1},...,Q^{\pi^1}_{\tau_1^\prime}$ have to be evaluated. Changes in quantities are denoted by \protect\redcircle, whereas unchanged quantities are denoted by \protect\greencircle. } 
    \label{fig:Figure_mu_Q}
\end{figure}

\begin{theorem}\label{th:optimal_sequence}
Among all periodic onto update sequences satisfying $\tau_{\ell+1}\neq \tau_\ell$, those that minimize the computational cost index \eqref{eq:periodic_cost_C} are exactly the sequences for which
$|\tau_{\ell+1}-\tau_\ell| = 1,\quad \text{for every } \ell \in \mathbb{N}$,
for any choice of weights $w_\mu$ and $w_Q$ with $w_\mu + w_Q > 0$.
Within this set of minimizing sequences, the following sequence is the unique one (up to cyclic shifts) with minimal period:
$$\tau = (0,1,\ldots,T\!-\!2,T\!-\!1,T\!-\!2,\ldots,1,0,1,\ldots \,).$$
In other words, the optimal periodic schedule with minimal period consists of a forward sweep $0,1,\dots,T-1$ followed by a backward sweep $T-1,T-2,\dots,0$, repeated indefinitely.
\end{theorem}
\begin{proof}
    Consider one period of the periodic sequence as $\tau_0,...,\tau_{M-1},\tau_M = \tau_0$, with the cost defined as in \eqref{eq:periodic_cost_C}.
    Let $S_+ := \sum_{\ell=0}^{M-1} \max(\tau_{\ell+1}-\tau_\ell,0)$ and $S_- := \sum_{\ell=0}^{M-1} \max(0,\tau_{\ell}-\tau_{\ell+1})$. By periodicity $\sum_{\ell=0}^{M-1} (\tau_{\ell+1} - \tau_\ell) = 0$, hence $S_+ = S_-$. Moreover, let $V := \sum_{\ell=0}^{M-1} | \tau_{\ell+1} - \tau_\ell | = S_+ + S_- = 2 S_+$. Then
    \begin{equation*}
        C = \frac{w_\mu S_+ + w_Q S_-}{M} = \frac{w_\mu + w_Q}{2} \cdot \frac{V}{M}
    \end{equation*}
    Thus, since $w_\mu + w_Q$,  minimizing $C$ is equivalent to minimizing the average variation $V/M$. Sequences that minimize the average variation are characterized by $|\tau_{\ell+1} - \tau_\ell| = 1$, leading to $V/M = 1$. In fact, if there exists an $\ell$ such that $|\tau_{\ell+1} - \tau_\ell| > 1$, then $V/M > 1$.
     Since the periodic sequence must be onto, the unique minimizer with minimal period $M$, is the monotone forward and backward sweep (and its cyclic shifts). 
\end{proof}


The specified optimal sequence with minimal period alternates between forward and backward sweeps. After each improvement step, only a single instance of either $\mu_t(s)$ or $Q^{\pi}_t(s,a)$ has to be recomputed. This sequence leads to two policy improvements per stage over a period. An optimal sequence with a larger period would have a different number of policy improvement for different stages, and thus favour some stages. We therefore pick the one with the smallest period.   This most efficient version of the general class of memoryless policy improvement algorithms provided in the previous section is summarized in Algorithm~\ref{alg:structured} and illustrated in Figure~\ref{fig:efficient_PI}.

\begin{figure}[b!]
    \centering
    \includegraphics[width=0.9\linewidth]{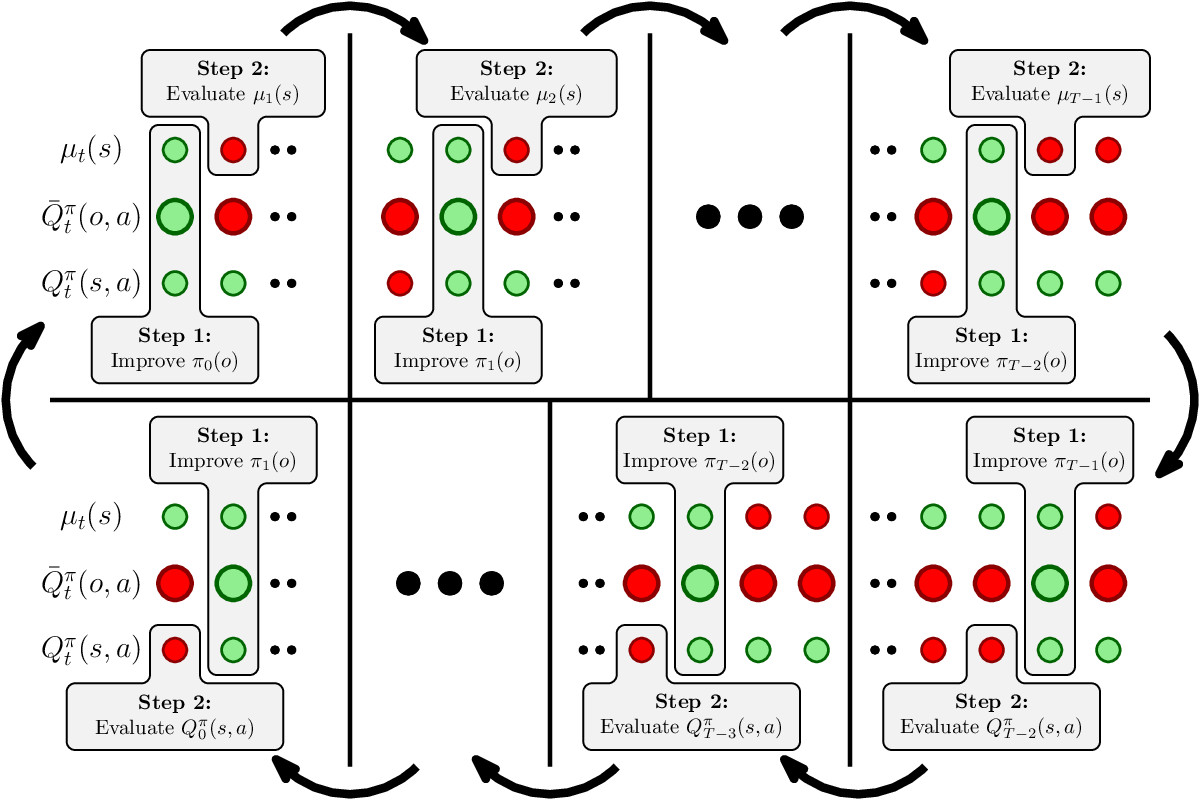}
    \caption{Illustration of Algorithm~\ref{alg:structured}, which alternates between a forward sweep (upper row), where after each improvement step $\mu_t(s)$ is recomputed, and a backward sweep (bottom row), where after each improvement step $Q^\pi_t(s,a)$ is recomputed.}
    \label{fig:efficient_PI}
\end{figure}

Consider again the forward and backward sweep at the end of Section~\ref{sec:III-A}. The cost $C$ for those two sequences with period $M=T$ is
\begin{equation*}
    C = \frac{w_\mu + w_Q}{2} \cdot \frac{2(T-1)}{T} = \frac{(T-1)(w_\mu + w_Q)}{T} \geq C^\ast
\end{equation*}
for $T \geq 2$. Any sequence that has jumps, such that $|\tau_{\ell+1} - \tau_\ell| \geq 1$, has a higher average variation than 1, and thus a larger cost. 


\begin{algorithm}[ht]
\caption{Efficient memoryless policy iteration for POMDPs}\label{alg:structured}
\KwIn{Initial policy $\pi^0$, $\mu_0$, $V_T$, $T$}

Use initial policy $\pi^0$ to compute $Q^{\pi^0}_t(s,a)$ with \eqref{eq:Q_eval_obs_based} for $t \in \{0,\dots,T-1\}$\;

\BlankLine
$\ell \gets 0$\;
{\tt{policystable} $\gets$ \tt{false}}\;
\BlankLine

\While{\tt{policystable = false}}{

    \BlankLine
    \emph{Forward sweep}
  
    \For{$t=0$ \KwTo $T-2$}{
        Compute $\bar{Q}^{\pi}_t(o,a)$ with \eqref{eq:Q_function_obs_based} and $\mu_{t}$\;
        $\pi_t^{\ell+1}(o) \gets \argmax_a \bar{Q}^{\pi}_t(o,a)$\;
        Compute $\mu_{t+1}$ with \eqref{eq:mu} using $\pi^{\ell+1}_t$ and $\mu_{t}$\;
        $\pi^{\ell+1} \gets \lbrace \pi_0^\ell,...,\pi_{t-1}^\ell,\pi_t^{\ell+1},\pi_{t+1}^\ell,\pi_{T-1}^\ell \rbrace$\;
        $\ell \gets \ell+1$\;
    }

    \BlankLine
    \emph{Backward sweep}
    
    \For{$t=T-1$ \KwTo $1$}{
        Compute $\bar{Q}^{\pi}_t(o,a)$ with \eqref{eq:Q_function_obs_based} and $\mu_{t}$\;
        $\pi^{\ell+1}_t(o) \gets \argmax_a \bar{Q}^{\pi}_t(o,a)$\;
        Compute $Q^{\pi^{\ell+1}}_{t-1}(s,a)$ with \eqref{eq:Q_eval_obs_based} using $\pi^{\ell+1}_t$ and $Q^{\pi^\ell}_{t}(s,a)$\;
        $\pi^{\ell+1} \gets \lbrace \pi_0^\ell,...,\pi_{t-1}^\ell,\pi_t^{\ell+1},\pi_{t+1}^\ell,\pi_{T-1}^\ell \rbrace$\;
        $\ell \gets \ell+1$\;
    }

    \BlankLine
    \If{$L^{\pi^\ell} = L^{\pi^{\ell-M}}$}{
        {\tt{policystable} $\gets$ \tt{true}}\;
    }
}
  \BlankLine
\KwOut{$\pi^{\ell}$}

\end{algorithm}

\clearpage
\section{Model-free Policy Iteration}\label{sec:IV}
 Inspired by the findings of the previous section, we propose in this section two algorithms that learn a memoryless policy from data under two different assumptions: the first algorithm assumes that data contains the state, whereas the second algorithm assumes that it only contains observations. As a consequence, the quantities that can be directly estimated in the state-informed method are $\alpha_t(s|o)$, $Q^{\pi}_t(s,a)$, and $q_t(o|s)$. For the observation-only method, it is not possible to directly trace the stationary distribution $\mu_t(s)$, and thus estimate $\alpha_t(s|o)$ from data. The only quantity that can be estimated is $\bar{Q}^{\pi}_t(o,a)$. Table~\ref{tab:supervised_unsupervised_summary} summarizes the assumptions, estimated quantities and consequences.
\par The state-informed setting is realistic in scenarios where training is performed in simulation or a high-fidelity digital twin, while deployment is restricted to partial observations. This is common in robotics and autonomous systems, where full simulator state (e.g., position, velocity, or contact forces) is available during learning but not at execution time. In such cases, the proposed method naturally fits a sim-to-real workflow. The output-informed setting appears in settings where the training data is the same as deployment data and contains only the output.
\par The state-informed method is discussed in Section~\ref{sec:IV-A}, whereas the observation-only method is presented in Section~\ref{sec:IV-B}. 

\begin{table}[b]
    \centering
    \caption{Summary of state-informed and observation-only model-free policy iteration.}
    \small
    \begin{tabular}{l m{3.4cm} m{4.1cm} m{3.6cm}}
        \textbf{Setting} & \textbf{Available Data} & \textbf{What can be estimated} & \textbf{Consequence} \\ \hline
        \textbf{State-informed} & States, observations, actions, rewards & $\alpha_t(s|o)$, $Q^{\pi}_t(s,a)$, $q_t(o|s)$ & Requires access to the state \\
        \textbf{Observation-only} & Observations, actions, rewards & $\bar{Q}^{\pi}_t(o,a)$ & Need to recollect data after each improvement
    \end{tabular}
    \label{tab:supervised_unsupervised_summary}
\end{table}

\subsection{State-informed model-free policy iteration}\label{sec:IV-A}
In line with the model-based memoryless policy iteration scheme in Section~\ref{sec:III-A}, we propose a model-free variant of this algorithm, again parametrized by an arbitrary periodic sequence $\tau_\ell \in \mathcal{T}$, $\ell \in \{0,1,\dots,\}$ with arbitrary period $M$. For now, we consider a data set that contains state information, such that $Q_t^{\pi}(s,a)$ and $\alpha_t(s|o)$ can be estimated directly.

\par While to obtain $\alpha_t(s|o)$ from $\mu_t(s)$ for a given policy $\pi^{\ell}$ it is desirable to have on-policy data, to obtain $Q_t^{\pi^\ell}(s,a)$ we need off-policy data.  Thus both quantities are estimated based on off-policy data and restricing the relevant on-policy data for computing $\pi^{\ell}$. We consider data sets $\mathcal{D}^\ell$, where the initial states are distributed according to the state distribution $\mu_0(s)$, and a  behaviour policy $b_t(a|o)$ (e.g. an epsilon-soft version of policy $\pi^\ell$) is applied afterwards. The data obtained with the behaviour policy is assumed to contain visits to all the state and control pairs a sufficient amount of times to obtain close predictions of  $Q_t^{\pi^\ell}(s,a)$.  Let these data sets be denoted by $\mathcal{D}^\ell := \lbrace \mathcal{D}^\ell_1,\mathcal{D}^\ell_2,...,\mathcal{D}^\ell_T \rbrace$. For each time $t$, there are $n_s$ data pairs:
\begin{equation*}
    \mathcal{D}^\ell_t := \left\lbrace (S_{t,i},O_{t,i},A_{t,i},R_{t,i}, S^\prime_{t,i}, O^\prime_{t,i} | i = 1,2,...,n_s \right\rbrace.
\end{equation*}
where $ S_{t,i}^\prime=S_{t+1,i}$, $ O_{t,i}^\prime=O^\prime_{t+1,i}$. This data set is sufficient for estimating $Q_t^{\pi^\ell}(s,a)$ and $\alpha_t(s|o)$. Before details on the estimation of these quantities are given, we define several variables that indicate the number of data samples in a data set:
\begin{align*}
    N_t(s) &= \sum_{i=1}^{n_s} \mathbbm{1}\lbrace S_{t,i}=s \rbrace &\quad
    N_t(s,o) &= \sum_{i=1}^{n_s} \mathbbm{1} \lbrace S_{t,i}=s, O_{t,i}=o\rbrace \\
    N_t(o) &= \sum_{i=1}^{n_s} \mathbbm{1}\lbrace O_{t,i}=o \rbrace &\quad
    N_t(s,a) &= \sum_{i=1}^{n_s} \mathbbm{1}\lbrace S_{t,i}=s,A_{t,i}=a \rbrace
\end{align*}
With these definitions, the observation model $q_t(o|s)$ can be estimated as $\hat{q}_t(o|s) = \frac{N_t(s,o)}{N_t(s)}$. Using bootstrapping techniques, $Q_t^{\pi}(s,a)$ can be estimated recursively, similar to~\eqref{eq:Q_eval_obs_based}, as
\begin{equation}\label{eq:Qhat}
    \hat{Q}^{\pi}_t(s,a) = \frac{1}{N_t(s,a)} \sum_{i=1}^{n_s} \left( \mathbbm{1}\lbrace S_{t,i}=s,A_{t,i}=a \rbrace \left( R_{t,i} + \hat{V}_{t+1}(S^\prime_{t,i}) \right)\right).
\end{equation}
where $\hat{V}_t(s) = \sum_o \hat{q}_t(o|s) \hat{Q}^\pi_t(s,\pi_t(o))$. 

In addition, $\alpha_t(s|o)$ can be estimated directly using importance sampling in a recursive manner, closely related to~\eqref{eq:mu}.
For the first stage, $\alpha_0(s|o)$ can be estimated as
\begin{equation}\label{eq:alpha0_hat}
    \hat{\alpha}_0(s|o) = \frac{N_0(s,o)}{N_0(o)}
\end{equation}
under the assumption that the data in the initial stage is distributed according to $\mu_0(s)$. 
The estimated corrected state observation counter $\tilde{N}_{t}(s,o)$ can be computed recursively as
\begin{equation*}
    \tilde{N}_{t+1}(s^\prime,o^\prime) = \frac{\sum_{i=1}^{n_s} \mathbbm{1}\lbrace A_{t,i} = \pi_t(O_{t,i}), S_{t,i}^\prime = s^\prime, O_{t,i}^\prime=o^\prime \rbrace \tilde{N}_{t}(S_{t,i},O_{t,i}) }{\sum_{i=1}^{n_s} \mathbbm{1}\lbrace A_{t,i} = \pi_t(O_{t,i})\rbrace}
\end{equation*}
where $\tilde{N}_{0}(s,o) = N_0(s,o)$. Based on these quantities, the posterior distribution $\alpha_t(s|o)$ can be estimated as
\begin{equation}\label{eq:alpha_hat}
    \hat{\alpha}_t(s|o) = \frac{\tilde{N}_{t}(s,o)}{\sum_s \tilde{N}_{t}(s,o)}.
\end{equation}

Combining the estimates of $Q^\pi_t(s,a)$ and $\alpha_{t}(s|o)$, the observation-action value can be computed according to
\begin{equation}\label{eq:Qhat_obs}
    \bar{Q}^{\pi}_t(o,a) = \sum_s \hat{Q}^{\pi}_t(s,a) \hat{\alpha}_{t}(s|o).
\end{equation}

These estimates allow us to apply the following state-informed model-free policy iteration scheme, which parallels the model-based method as proposed in Section~\ref{sec:III-A}.
\begin{enumerate}
  \item[] \hspace{-0.8cm} Given an initial deterministic policy $\pi^0$, and an $M$-periodic schedule $\lbrace \tau_\ell \rbrace$, set $\ell=0$.
  \item[0.]  \textbf{Collect data $\mathcal{D}^\ell$} Collect state, output and input trajectories using an $\epsilon$ soft version of policy $\pi^\ell$.
    \item[1.] \textbf{Policy evaluation:} Estimate $Q^{\pi^\ell}_t(s,a)$ and $\alpha_{t}(s|o)$ to compute $\bar{Q}^{\pi}_t(o,a)$ for $t =\tau_\ell$.
    \item[2.] \textbf{Policy improvement:} Compute
    \begin{equation*}
    {\pi}^{\ell+1}_t(o) = \argmax_a \bar{Q}_t^{\pi}(o,a) \text{ for } t=\tau_\ell.
    \end{equation*}
and set ${\pi}^{\ell+1}_t(o)={\pi}^{\ell}_t(o)$ for $t\neq \tau_\ell$. Set $\ell \rightarrow \ell+1$ and repeat $0$, $1$ and $2$ until the policy does not change after $M$ consecutive policy improvement steps.
\end{enumerate}

\par Recall from Section~\ref{sec:III-A} that there are several sequences by which the policy improvement steps can be executed in a policy iteration method, all leading to (possibly different) \textit{local} optima. An algorithm with the same efficient sequence of the model-based case alternating between forward and backward sweeps, is summarized in Algorithm~\ref{alg:model-free}. In the forward sweep, only $\alpha_t(s|o)$ is estimated using importance sampling, while $Q^\pi_t(s,a)$ remains unchanged. In the backward sweep $Q^\pi_t(s,a)$ is estimated while $\alpha_t(s|o)$ remains unchanged.

\begin{algorithm}[ht]
\caption{State-informed model-free policy iteration for POMDPs}\label{alg:model-free}
\KwIn{$\mathcal{D}^\ell$, $\pi^0$, $N_{\text{iter}}$}

Use $\pi^{0}$ to compute $\hat{Q}^{\pi^0}_t(s,a)$ with \eqref{eq:Qhat} for $t \in \{0,\dots,T-1\}$\;
Compute $\hat{\alpha}_0(s|o)$ with \eqref{eq:alpha0_hat}\;
Compute observation model $\hat{q}_t(o|s)$\;

\BlankLine
$\ell \gets 0$\;
\BlankLine

\For{$i=1$ \KwTo $N_{\mathrm{iter}}$}{

    \BlankLine
    \emph{Forward sweep}
  
    \For{$t=0$ \KwTo $T-2$}{
        Compute $\bar{Q}^\pi_t(o,a)$ with \eqref{eq:Qhat_obs}, using $\hat{Q}^\pi_t(s,a)$ and $\hat{\alpha}_t(s|o)$\;
        $\pi_t^{\ell+1}(o) \gets \argmax_a \bar{Q}^\pi_t(o,a)$\;
        Compute $\hat{\alpha}_{t+1}(s|o)$ with \eqref{eq:alpha_hat} using $\pi^{\ell+1}_t$ and $\hat{\alpha}_{t}(s|o)$\;
        $\pi^{\ell+1} \gets \lbrace \pi_0^\ell,...,\pi_{t-1}^\ell,\pi_t^{\ell+1},\pi_{t+1}^\ell,\pi_{T-1}^\ell \rbrace$\;
        $\ell \gets \ell+1$\;
    }

    \BlankLine
    \emph{Backward sweep}
    
    \For{$t=T-1$ \KwTo $1$}{
        Compute $\bar{Q}^\pi_t(o,a)$ with \eqref{eq:Qhat_obs}, using $\hat{Q}^\pi_t(s,a)$ and $\hat{\alpha}_t(s|o)$\;
        $\pi^{\ell+1}_t(o) \gets \argmax_a \bar{Q}^\pi_t(o,a)$\;
        Compute $\hat{Q}^{\pi^{\ell+1}}_{t-1}(s,a)$ with \eqref{eq:Qhat} using $\pi^{\ell+1}_t$ and $\hat{Q}^{\pi^\ell}_{t}(s,a)$\;
        $\pi^{\ell+1} \gets \lbrace \pi_0^\ell,...,\pi_{t-1}^\ell,\pi_t^{\ell+1},\pi_{t+1}^\ell,\pi_{T-1}^\ell \rbrace$\;
        $\ell \gets \ell+1$\;
    }
}
  \BlankLine
\KwOut{$\pi^{\ell}$}

\end{algorithm}

\subsection{Observation-only policy iteration}\label{sec:IV-B}
In the case only observations (and not state) information can be collected, the posterior distribution and the state-action value cannot be separated from each other. This however is a key feature to obtain an efficient policy iteration scheme. To still rely on the principles of monotonic convergence, developed in Section~\ref{sec:III}, the best one can do is to make only improvements at a single stage. The data sets used for observation-only model-free policy iteration are defined as
\begin{equation*}
    \bar{\mathcal{D}}^\ell_t := \left\lbrace (o_{t,i},a_{t,i},r_{t,i},o^\prime_{t,i})|i=1,2,...,n_s \right\rbrace.
\end{equation*}
This data set is generated by applying exploring actions at the stage considered for improvement, and on-policy actions of policy $\pi^\ell$ elsewhere. After the single-stage improvement, a complete new data set needs to be generated.
In this method, the observation-action values are estimated directly using Monte Carlo returns:
\begin{equation}\label{eq:Qhat_obs_only}
    \bar{Q}^\pi_t(o,a) = \frac{1}{N_t(o,a)} \sum_{i=1}^{n_s} \left( \mathbbm{1}\lbrace O_{t,i}=o,A_{t,i}=a \rbrace \sum_{k=t}^T R_{k,i} \right).
\end{equation}
where $N_t(o,a) = \sum_{i=1}^{n_s} \mathbbm{1}\lbrace O_{t,i}=o,A_{t,i}=a \rbrace$. 
The observation-only model-free policy iteration 
algorithm is omitted for the sake of brevity.






\section{Numerical Examples}\label{sec:V}
This section empirically evaluates the proposed memoryless policy iteration algorithms. We focus on three topics aligned with Sections~\ref{sec:III}~and~\ref{sec:IV}: (i) monotonic convergence behaviour, (ii) computational efficiency compared to existing methods, and (iii) performance in model-free settings. In Section~\ref{sec:V-A}, the proposed model-based algorithm of Section~\ref{sec:III} is compared with model-based policy gradient and exhaustive search baselines. We also provide a runtime comparison with recently introduced methods. In Section~\ref{sec:V-B}, the proposed model-free methods are analysed.

\subsection{Model-based policy iteration}\label{sec:V-A}
Random POMDPs are considered, which is common practice in the literature for these problems and allows us to easily scale the problem (see e.g. ~\cite{littman1995efficient}). The first example is a small-size environment, where $|\mathcal{S}| = 20$, $|\mathcal{A}| = 2$, $|\mathcal{O}| = 4$, and $T = 6$. The second example is a mid-size environment, where $|\mathcal{S}| = 40$, $|\mathcal{A}| = 10$, $|\mathcal{O}| = 20$, and $T = 20$. A larger environment with $T=50$, $|\mathcal{S}|=500$, and $|\mathcal{O}|=|\mathcal{A}|=100$ is considered later. The expected rewards $r_t(s,a)$ are also randomized. 

For both examples, we compare the proposed model-based algorithm with a model-based policy gradient method. For the policy gradient method, a stochastic soft-max policy is used, defined as
\begin{equation*}
    \pi_t(a|o,\theta) = \frac{e^{\theta_{o,a,t}}}{\sum_b e^{\theta_{o,b,t}}}
\end{equation*}
parametrized by $\theta_{o,a,t}$, $o \in \mathcal{O}$, $a \in \mathcal{A}$, $t \in \mathcal{T}$. The objective function which is maximized is $L^{\pi}$. Through the dependency of $L^\pi$ on parameters $\theta$, we are able to compute the gradient $\nabla_\theta L^{\pi}(\theta)$. To compute the gradient we need to sweep over stages as all values of $\theta$ affect $L^\pi$. A backtracking line search method is used to ensure monotonic convergence. For the line search we used the Armijo rule which entails evaluating $L^\pi$ a few times at each gradient step.  Given the relative low problem dimensions, computing the global optimal deterministic memoryless policy is still possible via exhaustive search. The number of possible policies is $|\mathcal{A}|^{|\mathcal{O}|\cdot T}$. In fact, for the small-size environment, the number of possible policies is $2^{4\cdot6} = 16.777.216$. For the mid-size environment however, this number is $10^{400}$, which makes the global optimal policy impossible to compute.

In Figure~\ref{fig:MB_example_small} the speed of convergence of the proposed model-based policy iteration method is compared with the policy gradient method and the exhaustive search method. For this particular small-size example, all methods converge to the same solution, which happens to be the global optimal policy.  The model-based policy iteration method converges to the optimal policy in 11 policy stage improvements. That implies that within one forward and backward sweep (as proposed in Section~\ref{sec:III-B}), no further improvements can be found. We count a stage improvement as a change of the action selected at a particular time $t$.  The model-based policy gradient method improves $T$ stages simultaneously with each gradient step. Thus, we have considered that each gradient step amounts to $T$ stage improvement steps. Note that both the gradient method and our memoryless policy iteration need a sweep over states and stages to improve all the stages and thus such a plot is sensible. A comparison using computation times is provided in the sequel. From Figure~\ref{fig:MB_example_small} it is clear that the policy gradient method requires more stage improvements than the model-based policy iteration method until convergence.

\begin{figure}[b!]
    \centering
    \includegraphics[width=0.8\linewidth]{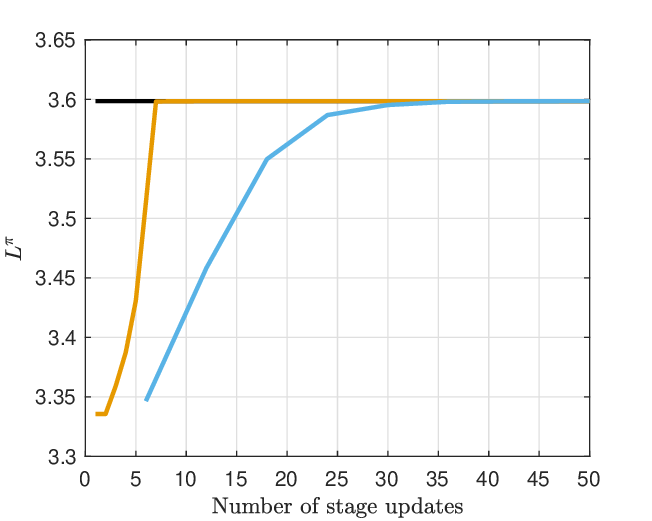}
    \caption{Random small-size POMDP, with $|\mathcal{S}| = 20$, $|\mathcal{A}| = 2$, $|\mathcal{O}| = 4$, and $T = 6$. Illustrated are the exhaustive search solution (\protect\blackline), the proposed model-based policy iteration method (\protect\orangeline), and model-based policy gradient (\protect\blueline).}
    \label{fig:MB_example_small}
\end{figure}

Figure~\ref{fig:MB_example_midsize} shows the evolution of the expected episodic return versus the number of stage improvements for a mid-size POMDP. When the size of the problem increases, the difference between the policy iteration and the policy gradient method becomes even more apparent. Again, both methods converge to the same locally optimal deterministic memoryless policy. The policy iteration method however found this deterministic policy within 45 unique stage improvements, whereas the policy gradient method, after 400 unique stage improvements (or 20 gradient steps) and it is only (very) close to this policy.

\begin{figure}[t!]
    \centering
    \includegraphics[width=0.8\linewidth]{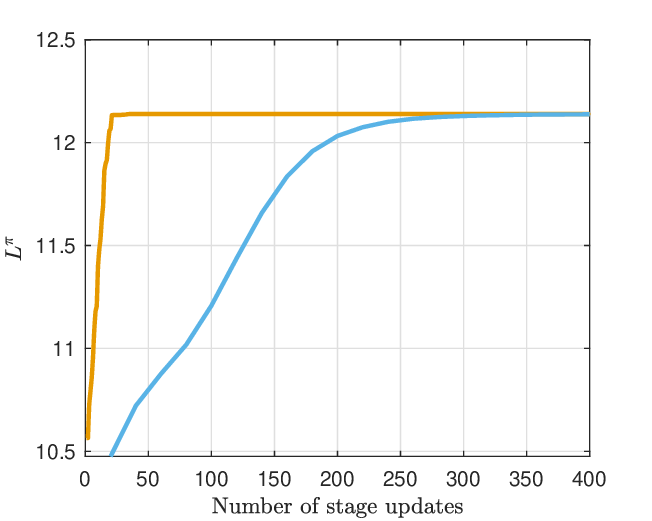}
    \caption{Random mid-size POMDP, with $|\mathcal{S}| = 40$, $|\mathcal{A}| = 10$, $|\mathcal{O}| = 20$, and $T = 20$. Illustrated are the proposed model-based policy iteration method (\protect\orangeline), and model-based policy gradient (\protect\blueline).}
    \label{fig:MB_example_midsize}
\end{figure}

We now compare the computation times of the two methods considered so far and also of the methods proposed in~\cite{cohen2023future},~\cite{MuellerMontufar2022}. For this we consider a POMDP, where $T=5$ and $|\mathcal{S}| = 20$. The number of observations coincide with the number actions and varied to compare computation times. All computations are done on the same computer.  For the Mixed Integer Linear Programming (MILP) method, presented in \cite{cohen2023future}, the Gurobi solver was used. The geometric method presented in \cite{MuellerMontufar2022} was solved using a standard sequential quadratic programming solver. All gradient-based methods were terminated when the improvements were below a certain threshold. The exhaustive search method and the MILP method are only computed for $|\mathcal{O}|=|\mathcal{A}| \in \lbrace 2,3 \rbrace$; larger values are too expensive to compute. The geometric method could be used without any problems up until $|\mathcal{O}|=|\mathcal{A}| = 16$.

\begin{figure}[tp]
    \centering
    \includegraphics[width=0.8\textwidth]{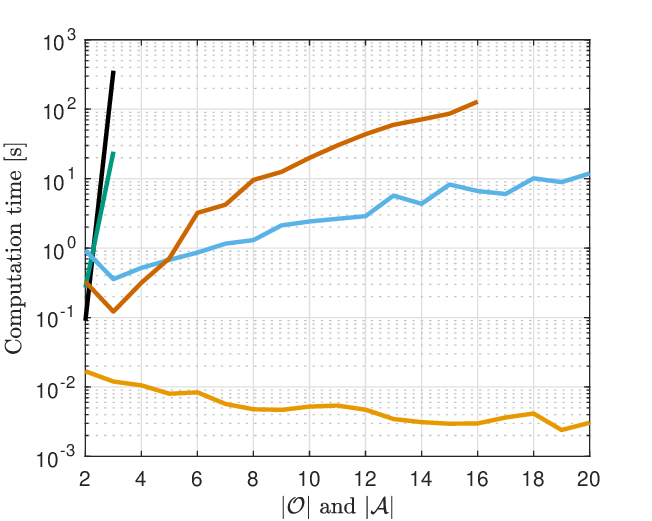}
    \caption{Comparison of computation time for the exhaustive search  method (\protect\blackline), the developed model-based policy iteration method (\protect\orangeline), a model-based policy gradient method (\protect\blueline), the mixed-integer linear programming method of \cite{cohen2023future} (\protect\greenline), and the geometric method of \cite{MuellerMontufar2022} (\protect\redline). $|\mathcal{O}|$ and $|\mathcal{A}|$ coincide, and are varied along the x-axis. Note the logarithmic scale on the y-axis.}
    \label{fig:Computation_time}
\end{figure}

The computation times are plotted in Figure~\ref{fig:Computation_time}. The policy iteration method significantly outperforms all the other methods.
Recall that both the geometric method, as well as the policy gradient method, are not guaranteed to output a deterministic policy, whereas the policy iteration method has that guarantee. The policy iteration method usually was able to find the local optimal policy within 2 or 3 forward + backward sweeps. To illustrate the optimality of the converged policy, consider the results in Table~\ref{tab:final_J_runtime}. The table clearly indicates that the policy found with the proposed policy iteration method are at least as good as the other methods for these examples. 

When the problem size increases, the policy iteration method remains superior to the policy gradient one. Consider e.g. a POMDP where $T=20$, $|\mathcal{S}|=200$, and $|\mathcal{O}|=|\mathcal{A}|=50$. The policy iteration method was able to solve it in $0.7$ seconds and achieve a local optimum where $L^\pi = 11.5880$. The policy gradient method on the other hand took $516$ seconds to achieve the same local optimum. 
Even for a large POMDP, where $T=50$, $|\mathcal{S}|=500$, and $|\mathcal{O}|=|\mathcal{A}|=100$, the policy iteration method found a local optimum within $15.06$ seconds.

\begin{table}[tp!]
    \centering
    \caption{$L^\pi$ after convergence to a (local) optimal policy.}
    \small
    \begin{tabular}{c|m{1.55cm} m{1.55cm} m{1.65cm} m{2.8cm} m{2.9cm} }
        $|\mathcal{O}|$ and $|\mathcal{A}|$ & \textbf{Policy iteration} & \textbf{Policy gradient} & \textbf{Exhaustive search} & \textbf{\cite{cohen2023future}} & \textbf{\cite{MuellerMontufar2022}}   \\ \hline
        2 & 3.2125 & 3.2125 & 3.2125 & 3.2125 & 3.2125 \\
        3 & 3.2140 & 3.2140 & 3.2140 & 3.1284 & 3.2133 \\
        4 & 3.3868 & 3.3868 & - & - & 3.3865 \\
        5 & 3.4920 & 3.4920 & - & - & 3.4866 \\
        6 & 3.6060 & 3.6059 & - & - & 3.6054 \\
        7 & 3.5749 & 3.5749 & - & - & 3.5745 \\
        8 & 3.6338 & 3.6338 & - & - & 3.6338 \\
        9 & 3.6392 & 3.6392 & - & - & 3.6364 \\
        10 & 3.6378 & 3.6378 & - & - & 3.6360 \\
    \end{tabular}
    \label{tab:final_J_runtime}
\end{table}

\subsection{Model-free}\label{sec:V-B}
Consider again a random POMDP, where $|\mathcal{S}| = 5$, $|\mathcal{A}| = 5$, $|\mathcal{O}| = 5$, and $T = 10$. Data is collected in batches of 5000 episodes per iteration. The two proposed methods in Section~\ref{sec:IV} are compared with LSPI~\cite{lagoudakis2003least} and policy gradient (REINFORCE algorithm~\cite{williams1992simple}). Except for the state-informed memoryless policy iteration method, all the methods used for comparison and the observation-based memoryless policy iteration only have access to observations. For the state-informed model-free policy iteration, a single forward and backward sweep is executed per iteration, whereafter new data is collected using an epsilon-soft policy, where $\epsilon = 0.5$. The LSPI method estimates the Q-function batch-wise using all data, whereas the policy gradient method makes a gradient step after each episode. The proposed model-free methods and LSPI give deterministic policies by definition, whereas for policy gradient the soft-max parametrization is used, leading to a stochastic policy.

The results are given in Figure~\ref{fig:model_free}, where the performance measure $L^{\pi}$ is averaged over 5 Monte Carlo simulations. The state-informed model-free policy iteration method leads to a memoryless, observation-based policy superior to the other methods. This result empirically supports the efficiency trade-offs discussed in Section~\ref{sec:IV}: separating estimation of $Q^\pi_t(s,a)$ and $\alpha_t(s|o)$ enables efficient multi-stage improvement from a single data set. The observation-only model-free policy iteration method takes approximately 10 iterations to come close to give comparable results, since then all stages have been improved at least once. For further iterations, it is rather consistent in finding an improved policy at each iteration. LSPI struggles to find a stable optimal policy in this partially observable setting, while policy gradient yields a comparably performing policy, albeit stochastic.

For the state-informed model-free policy iteration method, the data sets are generated with an epsilon-soft policy, where $\epsilon = 0.5$. By varying $\epsilon$, the data set can be completely on-policy ($\epsilon = 0$), fully exploratory ($\epsilon = 1$), or anything in between. To estimate $Q^\pi_t(s,a)$, it is preferred to have a fully exploratory policy, to cover all states and actions and get an accurate estimate for each pair. To estimate $\alpha_t(s|o)$, however, an on-policy data set is preferred, with more data samples distributed according the state distribution of the old policy. In Figure~\ref{fig:epsilon_vary}, the effect of $\epsilon$ on the convergence is shown. Neither the extreme of fully exploratory policy ($\epsilon = 1$), nor the almost fully on-policy ($\epsilon = 0.01$), are preferred. The best results are obtained with intermediate values of $\epsilon$.

\begin{figure}
    \centering
    \includegraphics[width=0.8\linewidth]{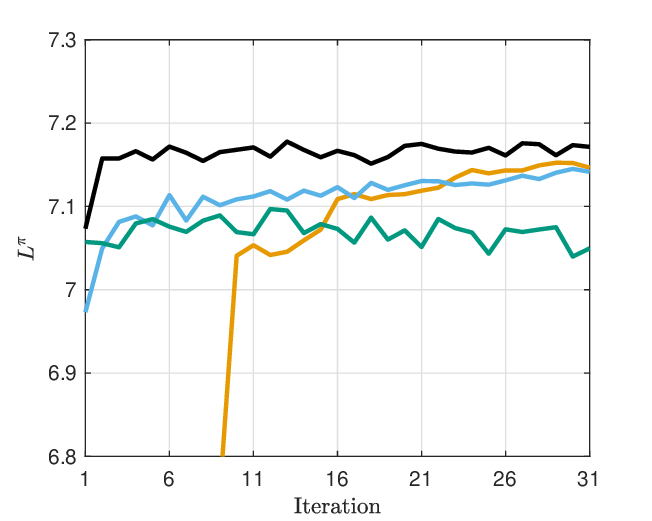}
    \caption{Comparison of state-informed model-free policy iteration (\protect\blackline), observation-only model-free policy iteration (\protect\orangeline), LSPI (\protect\greenline), and policy gradient (\protect\blueline). After each iteration, the exact expected episodic return function of the learned policy is computed. The results are averaged over 5 Monte Carlo simulations, with 5000 episodes per iteration.}
    \label{fig:model_free}
\end{figure}

\begin{figure}
    \centering
    \includegraphics[width=0.8\linewidth]{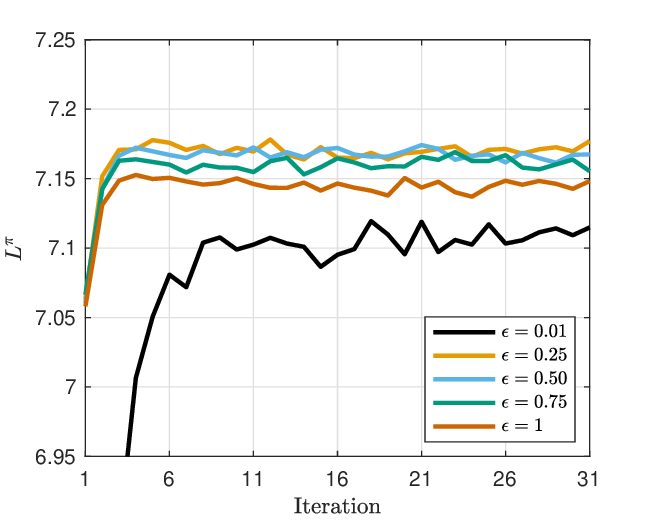}
    \caption{Effect of $\epsilon$ in epsilon-soft behaviour policy. The results are averaged over 20 Monte Carlo simulations, with 5.000 episodes per iteration.}
    \label{fig:epsilon_vary}
\end{figure}

\clearpage
\section{Conclusions and Future Work}\label{sec:VI}
This paper tackles the challenge of computing deterministic memoryless policies for episodic POMDPs from the perspective of policy iteration. We introduced a general class of memoryless policy iteration algorithms that perform single-stage output-based improvements according to a prescribed periodic schedule, and proved that any onto schedule yields monotonic improvement and convergence to a locally optimal memoryless policy. This establishes, for the episodic setting, a principled analogue of classical policy iteration that operates directly in the observation space despite the absence of the Markov property. Beyond monotonic convergence, we characterized the computational implications of different schedules and identified a unique (up to cyclic shift) schedule with minimal period that optimally balances forward and backward propagation of the state distribution and the state-action function.

The resulting algorithm retains the structure and interpretability of policy iteration while achieving substantial computational advantages over existing approaches, including policy gradient methods and recent specialized solvers for memoryless POMDP policies. Empirically, the proposed method consistently reaches local optima using much less computations.

Based on the findings on the model-based policy iteration method, two model-free policy iteration methods are proposed. The first one is state-informed model-free policy iteration, where, during learning, the state of the POMDP is part of the data set. This leads to a rather efficient policy iteration scheme, where a single off-policy data set can improve the policy at all stages. The second method is observation-only model-free policy iteration, where there is no access to the underlying state; only observations are in the data set. In this method, after each single-stage policy improvement step, a new on-policy data set has to be gathered. The proposed model-free policy iteration methods are compared with LSPI and policy gradient. The state-informed method significantly outperforms the other methods with the same data set size. The observation-only method achieves eventually a better performing policy than conventional methods, although the rate of convergence is rather slow and thus a large amount of data is required before a satisfying performance level is achieved.

Directions for future research include extensions to infinite horizon and continuous state, action and observation spaces. Furthermore, extensions to approximate policy iteration methods will allow handling much larger problems. 


\acks{The research is carried out as part of the ITEA4 20216 ASIMOV project. The ASIMOV activities are supported by the Netherlands Organisation for Applied Scientific Research TNO and the Dutch Ministry of Economic Affairs and Climate (project number: AI211006). The research leading to these results is partially funded by the German Federal Ministry of Education and Research (BMBF) within the project ASIMOV-D under grant agreement No. 01IS21022G [DLR], based on a decision of the German Bundestag.}


\clearpage
\vskip 0.2in
\bibliography{sample}

\newpage
\appendix
\section{Proof of Theorem~\ref{th:monotonic_convergence}}
We first note that $L^\pi$ can be written as:
\begin{align}
    L^{\pi} &= \mathbb{E}\left\lbrack \sum_{t=0}^{T-1} R_t + V_T(s_T) \right\rbrack \nonumber \\
    &= \sum_o \gamma_0(o) \sum_a \pi_0(a|o) \bar{Q}_0^{\pi}(o,a) \label{eq:L_def_2} \\
    &= \sum_s \mu_0(s) \sum_o q_0(o|s) \sum_a \pi_0(a|o) Q_0^{\pi}(s,a) \nonumber
\end{align}
If $\tau_{\ell+1} = 0$, it is obvious that
\begin{align*}
    L^{\pi^{\ell+1}} &= \sum_o \gamma_0(o) \sum_a \pi^{\ell+1}_0(a|o) \bar{Q}_0^{\pi^{\ell+1}}(o,a)  \\
    &\geq \sum_o \gamma_0(o) \sum_a \pi^\ell_0(a|o) \bar{Q}_0^{\pi^\ell}(o,a)  \\
    &= L^{\pi^{\ell}}
\end{align*}
since $\gamma_0(o)$ is independent of the policy applied at $t=0$ and $\bar{Q}_0^{\pi^{\ell+1}}(o,a) = \bar{Q}_0^{\pi^{\ell}}(o,a)$ only depends on the policy applied at $t>0$. 
If the policy is improved at $\tau_{\ell+1} \neq 0$, we can unroll \eqref{eq:L_def_2} up until $k=\tau_{\ell+1}$, such that we obtain:
\begin{align}
    L^{\pi^{\ell+1}} &= \sum_{i=0}^{k-1} \sum_s \mu_i(s) \sum_o q_i(o|s) \sum_a \pi^{\ell+1}_i(a|o) r_i(s,a) \nonumber \\
    &\quad\quad\quad + \sum_{s^\prime} \mu_k(s^\prime) \sum_{o^\prime} q_k(o^\prime|s^\prime) \sum_{a^\prime} \pi^{\ell+1}_{k}(a^\prime|o^\prime) \bar{Q}^{\pi^{\ell+1}}_{k}(s^\prime,a^\prime) \nonumber\\
    &= \sum_{i=0}^{k-1} \sum_s \mu_i(s) \sum_o q_i(o|s) \sum_a \pi^{\ell+1}_i(a|o) r_i(s,a) + \sum_o \gamma_k(o) \sum_a \pi^{\ell+1}_k(a|o) \bar{Q}^{\pi^{\ell+1}}_k(o,a) \nonumber\\
    &= \sum_{i=0}^{k-1} \sum_s \mu_i(s) \sum_o q_i(o|s) \sum_a \pi^{\ell}_i(a|o) r_i(s,a) + \sum_o \gamma_k(o) \sum_a \pi^{\ell+1}_k(a|o) \bar{Q}^{\pi^{\ell+1}}_k(o,a) \nonumber\\
    &\geq \sum_{i=0}^{k-1} \sum_s \mu_i(s) \sum_o q_i(o|s) \sum_a \pi^\ell_i(a|o) r_i(s,a) + \sum_o \gamma_k(o) \sum_a \pi^\ell_k(a|o) \bar{Q}^{\pi^\ell}_k(o,a) \label{eq:th_inequality}\\
    &= L^{\pi^{\ell}}. \nonumber
\end{align}
This follows from the fact that the policy remains unchanged for $t < \tau_{\ell+1}$: $\pi^{\ell+1}_i = \pi^{\ell}_i$ for $i \in \lbrace 0,...,\tau_{\ell+1}-1\rbrace$. Furthermore, $\bar{Q}_k^{\pi^{\ell+1}}(o,a) = \bar{Q}_k^{\pi^{\ell}}(o,a)$, since these values are independent on the policy at $k = \tau_{\ell+1}$. The inequality in \eqref{eq:th_inequality} follows directly from the optimization problem. 
The last equality in \eqref{eq:th_inequality} only holds when the values of $\bar{Q}_k^{\pi^\ell}(o,a)$ are aligned with (old) policy $\pi^\ell$, which the class of algorithms ensures each improvement step. Therefore each improvement step in the class of algorithms monotonically increases the expected episodic return. 
    
To continue the proof of convergence to a local optimal policy, let the expected episodic return after the $j$'th improvement be denoted by $L^{\pi^{(j)}}$. Then,
\begin{equation*}
    L^{\pi^{(0)}} \leq  L^{\pi^{(1)}} \leq ... \leq L^{\pi^\ast}
\end{equation*}
This sequence is bounded from above by the global optimal policy $\pi^\ast$. Let the period of $\lbrace \tau_\ell \rbrace$ be given as $M$. Since the number of policies is finite and the cost is upper and lower bounded and decreases monotonically, the cost must converge to a limit value. In particular, there exists  $m$ such that
\begin{equation*}
    L^{\pi^{(m)}} = L^{\pi^{(m+1)}} = ... = L^{\pi^{(m+M)}}
\end{equation*}
for the given $M$. Since the sequence $\{\tau_\ell\}$ is onto, no single stage can be improved further, and thus a local optimal policy has been found.

\end{document}